\documentclass[10pt]{article}

\usepackage{charter} 
\usepackage{fullpage}
\usepackage{titlesec}
\usepackage{amssymb} 
\usepackage{amsfonts}
\usepackage{amsmath} 
\usepackage{cite}
\usepackage{bm} 
\usepackage{graphicx} 
\usepackage[usenames, dvipsnames]{color}
\usepackage{cite} 
\usepackage{authblk} 
\usepackage{algorithm}
\usepackage{algpseudocode}
\usepackage{float}
\usepackage{url}
\usepackage{mathtools}
\usepackage{multirow}
\usepackage{booktabs}
\usepackage{enumitem}
\usepackage{bbm}
\usepackage{amsthm}
\usepackage{tabularx}

\titleformat*{\section}{\large\bfseries}
\titleformat*{\subsection}{\normalsize\bfseries}
\titleformat*{\subsubsection}{\small\bfseries}

\newcolumntype{L}[1]{>{\raggedright\arraybackslash}p{#1}}
\newcolumntype{C}[1]{>{\centering\arraybackslash}p{#1}}
\newcolumntype{R}[1]{>{\raggedleft\arraybackslash}p{#1}}

\def\defn{\,\triangleq\,}
\def\sgn{\mathsf{sgn}}

\def\uin{{u_{\text{\tiny in}}}}

\def\ubfin{{\mathbf{u}_{\text{\tiny in}}}}

\def\rbm{{\bm{r}}}

\def\sbm{{\bm{s}}}
\def\xbm{{\bm{x}}}
\def\ebm{{\bm{e}}}
\def\ybm{{\bm{y}}}
\def\zbm{{\bm{z}}}
\def\zerobm{{\bm{0}}}

\def\Ibm{{\bm{I}}}
\def\Hbm{{\bm{H}}}
\def\Dbm{{\bm{D}}}

\def\Isf{\mathsf{I}}
\def\Psf{\mathsf{P}}
\def\Tsf{\mathsf{T}}

\def\Fsf{\mathsf{F}}
\def\Gsf{\mathsf{G}}
\def\Psfhat{\widehat{\mathsf{P}}}

\def\xbmhat{{\widehat{\bm{x}}}}

\def\nablahat{{\hat{\nabla}}}

\def\xbmast{{\bm{x}^\ast}}
\def\zbmast{{\bm{z}^\ast}}
\def\sbmast{{\bm{s}^\ast}}

\def\Sbf{{\mathbf{S}}}

\def\R{\mathbb{R}}
\def\N{\mathbb{N}}
\def\E{\mathbb{E}}

\def\C{\mathbb{C}}

\def\argmin{\mathop{\mathsf{arg\,min}}}
\def\proposed{\text{PnP-SGD}}

\def\fix{\mathsf{fix}}

\def\prox{\mathsf{prox}}
\def\denoise{\mathsf{denoise}}

\def\Dsf{\mathsf{D}}

\newtheorem{proposition}{Proposition}
\newtheorem{definition}{Definition}

\newtheorem{corollary}{Corollary}
\newtheorem{assumption}{Assumption}

\begin{document}

\title{\Large An Online Plug-and-Play Algorithm for Regularized Image Reconstruction}

{\normalsize\author{Yu~Sun$^{\footnotesize 1}$, Brendt~Wohlberg$^{\footnotesize 2}$,~and~Ulugbek~S.~Kamilov$^{\footnotesize 1, 3, \ast}$\\
\emph{\small $^{\footnotesize 1}$Department of Computer Science and Engineering,~Washington University in St.~Louis, MO 63130, USA}\\
\emph{\small $^{\footnotesize 2}$Los Alamos National Laboratory, Theoretical Division, Los Alamos, NM 87545 USA}\\
\emph{\small $^{\footnotesize 3}$Department of Electrical and Systems Engineering,~Washington University in St.~Louis, MO 63130, USA}\\
\small$^{\footnotesize *}$\emph{Email}: \texttt{kamilov@wustl.edu}
}}

\markboth{An Online Plug-and-Play Algorithm for Regularized Image Reconstruction}%
{Kamilov:  An Online Plug-and-Play Algorithm for Regularized Image Reconstruction}

\date{}
\maketitle 

\begin{abstract}
Plug-and-play priors (PnP) is a powerful framework for regularizing imaging inverse problems by using advanced denoisers within an iterative algorithm. Recent experimental evidence suggests that PnP algorithms achieve state-of-the-art performance in a range of imaging applications. In this paper, we introduce a new online PnP algorithm based on the iterative shrinkage/thresholding algorithm (ISTA). The proposed algorithm uses only a subset of measurements at every iteration, which makes it scalable to very large datasets. We present a new theoretical convergence analysis, for both batch and online variants of PnP-ISTA, for denoisers that do not necessarily correspond to proximal operators. We also present simulations illustrating the applicability of the algorithm to image reconstruction in diffraction tomography. The results in this paper have the potential to expand the applicability of the PnP framework to very large and redundant datasets.
\end{abstract}



\section{Introduction}
\label{Sec:Intro}

The reconstruction of an unknown image $\xbm \in \R^n$ from a set of noisy measurements $\ybm \in \R^m$ is one of the most widely studied problems in computational imaging. The task is frequently formulated as an optimization problem
\begin{equation}
\label{Eq:RegularizedOptimization}
\xbmhat = \argmin_{\xbm \in \R^N} \left\{f(\xbm)\right\} \quad\text{with}\quad f(\xbm) = d(\xbm) + r(\xbm),
\end{equation}
where $d$ is the data-fidelity term that penalizes the mismatch to the measurements and $r$ is the regularizer that imposes prior knowledge regarding the unknown image. Some popular imaging priors include nonnegativity, transform-domain sparsity, and self-similarity~\cite{Rudin.etal1992, Figueiredo.Nowak2001, Elad.Aharon2006, Danielyan.etal2012}.

Over the past two decades, a substantial effort has been devoted to combining the best regularizers with efficient optimization algorithms. The large dimensionality of the imaging data and nondifferentiability of many regularizers has led to the widespread adoption of proximal algorithms~\cite{Parikh.Boyd2014}---such as variants of iterative shrinkage/thresholding algorithm (ISTA)~\cite{Figueiredo.Nowak2003, Daubechies.etal2004, Bect.etal2004, Beck.Teboulle2009a} and alternating direction method of multipliers (ADMM)~\cite{Eckstein.Bertsekas1992, Afonso.etal2010, Boyd.etal2011}. These algorithms avoid differentiating the regularizer by using a mathematical concept known as the proximal operator, which is itself an optimization problem equivalent to regularized image denoising.

The mathematical equivalence of the proximal operator to denoising has recently inspired Venkatakrishnan \emph{et al.}~\cite{Venkatakrishnan.etal2013} to introduce the powerful plug-and-play priors (PnP) framework for image reconstruction. The key idea in PnP is to replace the proximal operator in an iterative algorithm with a state-of-the-art image denoiser, such as BM3D~\cite{Dabov.etal2007}, WNNM~\cite{Gu.etal2014}, or TNRD~\cite{Chen.Pock2016}, which does not necessarily have a corresponding regularization objective. This implies that PnP methods generally lose interpretability as optimization problems. Nonetheless, the framework has gained in popularity due to its effectiveness in a range of applications in the context of imaging inverse problems~\cite{Sreehari.etal2016, Chan.etal2016, Brifman.etal2016, Teodoro.etal2016, Zhang.etal2017a, Teodoro.etal2017, Ono2017, Meinhardt.etal2017, Kamilov.etal2017}. In particular, the effectiveness of PnP was demonstrated beyond the original ADMM formulation~\cite{Venkatakrishnan.etal2013} to other proximal algorithms such as primal-dual splitting and ISTA~\cite{Ono2017, Meinhardt.etal2017, Kamilov.etal2017}.

All current PnP algorithms are iterative \emph{batch} procedures, which means that they use the full set of measurements at every iteration. This effectively precludes their application to very large datasets~\cite{Bottou.Bousquet2007} common in three-dimensional (3D) imaging or in imaging of dynamic objects~\cite{Kamilov.etal2016, Degraux.etal2017}. In this paper, we address this limitation by proposing a new \emph{online} extension for PnP-ISTA. By using only a subset of the measurements at a time, the proposed algorithm scales to datasets that would otherwise be prohibitively large for batch processing. More specifically, the key contributions of this paper are as follows.
\begin{itemize}
\item We present a detailed theoretical convergence analysis of batch PnP-ISTA under a set of explicit assumptions. Our analysis complements the recent theoretical results on PnP-ADMM by Sreehari \emph{et al.}~\cite{Sreehari.etal2016} and Chan \emph{et al.}~\cite{Chan.etal2016} in two major ways. We show that for PnP-ISTA the symmetric gradient assumption from~\cite{Sreehari.etal2016} is not necessary, while the bounded denoiser assumption from~\cite{Chan.etal2016} is not sufficient to establish the convergence.

\begin{figure*}
\begin{minipage}[t]{.5\textwidth}
\begin{algorithm}[H]
\caption{$\mathsf{ISTA}$}\label{alg:ista}
\begin{algorithmic}[1]
\State \textbf{input: } $\xbm^0 = \sbm^0 \in \R^n$, $\gamma > 0$, and $\{q_k\}_{k \in \N}$
\For{$k = 1, 2, \dots$}
\State $\zbm^k \leftarrow \sbm^{k-1}-\gamma \nabla d(\sbm^{k-1})$
\State $\xbm^k \leftarrow \prox_{\gamma r}(\zbm^k)$
\State $\sbm^k \leftarrow \xbm^k + ((q_{k-1}-1)/q_k)(\xbm^k-\xbm^{k-1}) $
\EndFor\label{euclidendwhile}
\end{algorithmic}
\end{algorithm}%
\end{minipage}
\hspace{0.25em}
\begin{minipage}[t]{.5\textwidth}
\begin{algorithm}[H]
\caption{$\mathsf{ADMM}$}\label{alg:admm}
\begin{algorithmic}[1]
\State \textbf{input: } $\xbm^0 \in \R^n$, $\sbm^0 = \zerobm$, and $\gamma > 0$
\For{$k = 1, 2, \dots$}
\State $\zbm^k \leftarrow \prox_{\gamma d}(\xbm^{k-1} - \sbm^{k-1})$
\State $\xbm^k \leftarrow \prox_{\gamma r}(\zbm^k + \sbm^{k-1})$
\State $\sbm^k \leftarrow \sbm^{k-1} + (\zbm^k - \xbm^k)$
\EndFor\label{euclidendwhile}
\end{algorithmic}
\end{algorithm}%
\end{minipage}
\end{figure*}

\item We extend the traditional batch PnP framework with our novel online algorithm called $\proposed$. We prove the theoretical convergence of the algorithm to the same set of fixed points as batch PnP-ISTA and PnP-ADMM. This makes $\proposed$ a powerful and theoretically sound alternative for large-scale image reconstruction. We also illustrate its applicability with several numerical simulations on image reconstruction problems encountered in diffraction tomography~\cite{Kak.Slaney1988}.

\end{itemize}


\section{Background}
\label{Sec:Background}

In this section, we provide the background material that forms the foundation to our contributions. We first review the problem of regularized image reconstruction and then introduce more recent results related to the PnP algorithms.

\subsection{Inverse problems in imaging}

Consider the linear inverse problem
\begin{equation}
\ybm = \Hbm\xbm+ \ebm,
\end{equation}
where the goal is to recover $\xbm \in \R^n$ given the measurements $\ybm \in \R^m$. Here, the measurement matrix $\Hbm \in \R^{m \times n}$ models the response of the imaging system and $\ebm \in \R^m$ represents the measurement noise, which is often assumed to be independent and identically distributed (i.i.d.) Gaussian. When the inverse problem is nonlinear, the measurement operator can be generalized to a more general mapping ${\Hbm: \R^{n} \rightarrow \R^m}$ with $\ybm = \Hbm(\xbm) + \ebm$.

Practical inverse problems are often ill-posed, which often leads to the formulation in~\eqref{Eq:RegularizedOptimization}. In such cases, one of the most popular data-fidelity terms is least-squares
\begin{equation}
\label{Eq:LeastSquares}
d(\xbm) = \frac{1}{2}\|\ybm - \Hbm\xbm\|_2^2,
\end{equation}
which imposes an $\ell_2$-penalty on data-fit. Similarly, two common regularizers for images include the spatial sparsity-promoting penalty $r(\xbm) \defn \lambda \|\xbm\|_{1}$ and TV penalty ${r(\xbm) \defn \lambda\|\Dbm\xbm\|_{1}}$, where $\lambda > 0$ is the regularization parameter and $\Dbm$ is the discrete gradient operator~\cite{Rudin.etal1992, Tibshirani1996, Candes.etal2006, Donoho2006}.

Many popular regularizers, such as the ones based on the $\ell_1$-norm, are nondifferentiable. Two common algorithms for working with such regularizers are ISTA and ADMM summarized in Algorithm~\ref{alg:ista} and~\ref{alg:admm}, respectively. The key step for handling nonsmooth regularizers is the proximal operator~\cite{Moreau1965}
\begin{equation}
\label{Eq:ProximalOperator}
\prox_{\gamma r}(\zbm) \defn \argmin_{\xbm \in \R^n} \left\{\frac{1}{2}\|\xbm-\zbm\|_2^2 + \gamma r(\xbm)\right\}.
\end{equation}
According to definition~\eqref{Eq:ProximalOperator}, the proximal operator corresponds to an image denoiser formulated as regularized optimization. Note also that when the values for $\{q_k\}$ in Algorithm~\ref{alg:ista} are adapted as
\begin{equation}
\label{Eq:Accelerate}
q_k \leftarrow \frac{1}{2}\left(1+\sqrt{1+4q_{k-1}^2}\right)
\end{equation}
the algorithm corresponds to the accelerated variant of ISTA, known as fast ISTA (FISTA)~\cite{Beck.Teboulle2009}. On the other hand, when $q_k = 1$ for all $k \in \N$, then one recovers the traditional ISTA. In this paper, we will use the term ISTA to refer to both algorithms, with an understanding that the selection of $\{q_k\}$ acts as a switch between the methods.

A careful inspection of ISTA and ADMM reveals a fundamental conceptual difference between the algorithms in their treatment of the data-fidelity. While ISTA relies on the gradient $\nabla d$, ADMM relies on the proximal operator $\prox_{\gamma d}$. For a large class of linear and nonlinear inverse problems, the gradient of the data-fidelity is significantly easier to evaluate compared to its proximal operator. As an example, for least-squares we have
\begin{equation}
\nabla d(\xbm) = \Hbm^\Tsf(\Hbm\xbm- \ybm)
\end{equation}
and
\begin{subequations}
\label{Eq:ADMMDataUpdate}
\begin{align}
\prox_{\gamma d}(\xbm) &= \argmin_{\zbm \in \R^n} \left\{\frac{1}{2}\|\zbm-\xbm\|_2^2 + \frac{\gamma}{2}\|\Hbm\zbm-\ybm\|_2^2\right\} \\
&= [\Ibm + \gamma \Hbm^\Tsf\Hbm]^{-1}(\xbm+\gamma \Hbm^\Tsf\ybm).
\end{align}
\end{subequations}
The matrix inversion in~\eqref{Eq:ADMMDataUpdate} can make ADMM updates computationally expensive for problems where the measurement matrix is not easily invertible.

The theoretical analysis in this paper is closely related to the convergence results established for first-order methods by Nesterov~\cite{Nesterov2004} and Beck and Teboulle~\cite{Beck.Teboulle2009}. In particular, our work is related to inexact proximal-gradient optimization that was extensively investigated by several researchers~\cite{Zinkevich2003, Duchi.Singer2009, Schmidt.etal2011, Bertsekas2011, Devolder.etal2013, Kamilov.etal2014a, Ghadimi.Lan2016, Kamilov2017}. We extend this prior work beyond traditional optimization, where denoising operators do not necessarily correspond to proximal operators of a given objective. To achieve this, we adopt the monotone operator theory~\cite{Bauschke.Combettes2010, Ryu.Boyd2016}, which enables a unified analysis of PnP methods by expressing them as finding zeros of some operator.

\subsection{Using denoisers as priors}

\begin{figure*}
\begin{minipage}[t]{.5\textwidth}
\begin{algorithm}[H]
\caption{$\mathsf{PnP}$-$\mathsf{ISTA}$}\label{alg:pppista}
\begin{algorithmic}[1]
\State \textbf{input: } $\xbm^0 = \sbm^0 \in \R^n$, $\gamma > 0$, $\sigma > 0$, and $\{q_k\}_{k \in \N}$
\For{$k = 1, 2, \dots$}
\State $\zbm^k \leftarrow \sbm^{k-1}-\gamma \nabla d(\sbm^{k-1})$
\State $\xbm^k \leftarrow \denoise_\sigma(\zbm^k)$
\State $\sbm^k \leftarrow \xbm^k + ((q_{k-1}-1)/q_k)(\xbm^k-\xbm^{k-1}) $
\EndFor\label{euclidendwhile}
\end{algorithmic}
\end{algorithm}%
\end{minipage}
\hspace{0.25em}
\begin{minipage}[t]{.5\textwidth}
\begin{algorithm}[H]
\caption{$\mathsf{PnP}$-$\mathsf{ADMM}$}\label{alg:pppadmm}
\begin{algorithmic}[1]
\State \textbf{input: } $\xbm^0 \in \R^n$, $\sbm^0 = \zerobm$, $\gamma > 0$, and $\sigma > 0$
\For{$k = 1, 2, \dots$}
\State $\zbm^k \leftarrow \prox_{\gamma d}(\xbm^{k-1} - \sbm^{k-1})$
\State $\xbm^k \leftarrow \denoise_\sigma(\zbm^k + \sbm^{k-1})$
\State $\sbm^k \leftarrow \sbm^{k-1} + (\zbm^k - \xbm^k)$
\EndFor\label{euclidendwhile}
\end{algorithmic}
\end{algorithm}%
\end{minipage}
\end{figure*}

Both ISTA and ADMM have modular structures in the sense that  the prior on the image is only imposed via the proximal operator. Additionally, since the proximal operator is mathematically equivalent to regularized image denoising, the powerful idea of Venkatakrishnan \emph{et al.}~\cite{Venkatakrishnan.etal2013} was to consider replacing it with a more general denoising operator $\denoise_\sigma(\cdot)$ of controllable strength ${\sigma > 0}$. In order to be backward compatible with the traditional optimization formulation, this strength parameter is often scaled with the step-size as $\sigma = \sqrt{\gamma \lambda}$, for some parameter $\lambda > 0$.

The original formulation of PnP~\cite{Venkatakrishnan.etal2013} relies on ADMM. However, recent results have shown that it can be as effective when used with other proximal algorithms~\cite{Ono2017, Meinhardt.etal2017, Kamilov.etal2017} or with another class of algorithms known as approximate message passing (AMP)~\cite{Metzler.etal2016, Metzler.etal2016a, Fletcher.etal2018}. AMP-based algorithms have been shown to be effective for problems where $\Hbm$ is large and random~\cite{Donoho.etal2009, Bayati.Montanari2011}, but are also known to be unstable for general matrices $\Hbm$~\cite{Caltagirone.etal2014, Rangan.etal2014, Rangan.etal2017}. Therefore, in this paper, our focus will be exclusively on the variants of PnP based on ISTA and ADMM, summarized in Algorithm~\ref{alg:pppista} and~\ref{alg:pppadmm}, respectively.

Several recent publications have analyzed the theoretical convergence of PnP algorithms~\cite{Sreehari.etal2016, Chan.etal2016, Meinhardt.etal2017, Teodoro.etal2017}. Sreehari \emph{et al.}~\cite{Sreehari.etal2016} have established the convergence of PnP-ADMM to the global minimum of some implicitly defined objective function. Specifically, by building on the theoretical analysis by Moreau~\cite{Moreau1965}, they show that $\denoise_\sigma$ is a valid proximal operator of some implicit regularizer if it is nonexpansive and $\nabla \denoise_\sigma(\xbm)$ is a symmetric matrix for all ${\xbm \in \R^n}$. Chan \emph{et al.}~\cite{Chan.etal2016} have proved a fixed-point convergence of PnP-ADMM for bounded denoisers, which are defined as denoisers satisfying
\begin{equation}
\label{Eq:BoundedDenoiser}
\frac{1}{n}\|\denoise_\sigma(\xbm)-\xbm\|_2^2 \leq \sigma^2 \, c,
\end{equation}
for any $\xbm \in \R^n$, where $c > 0$ is a constant independent of $n$ and $\sigma$. Meinhardt \emph{et al.}~\cite{Meinhardt.etal2017} have shown that for continuous denoisers several PnP algorithms admit an equivalent fixed-point iteration. More recently, Teodoro \emph{et al.}~\cite{Teodoro.etal2017} considered a special class of denoisers based on Gaussian mixture models (GMMs) and showed that PnP-ADMM converges when the GMM denoiser is simplified to be a linear function of its input.

A different but related approach to denoiser-driven regularization was recently proposed by Romano \emph{et al.}~\cite{Romano.etal2017}. They proposed the regularization by denoising (RED) framework, where an explicit regularizer is constructed as
\begin{equation}
r(\xbm) = \frac{1}{2}\xbm^\Tsf(\xbm-\denoise_\sigma(\xbm)).
\end{equation}
Remarkably, they also showed that under some conditions, the gradient of the regularizer has a very simple expression. More recently, Reehorst and Schniter~\cite{Reehorst.Schniter2018} have provided additional insight into RED by establishing conditions for the existence of explicit regularizers based on denoising operators. The key difference between PnP and RED is that the former does not seek to define an explicit regularization functional, but relies on the fixed points of a given denoising operator for regularization. This generality of the PnP framework makes it widely applicable, but also substantially complicates its theoretical analysis.

Another recent related framework is the consensus equilibrium (CE) by Buzzard \emph{et al.}~\cite{Buzzard.etal2017}. Given multiple sources of information (defined via image denoisers or other similar mappings), CE proposes to fuse them by computing a specific equilibrium point. The CE framework extends the traditional consensus optimization~\cite{Boyd.etal2011} to operators that are not necessarily proximal operators and formulates a new variant of PnP that can handle multiple denoising functions. In this paper, we will restrict our attention to the traditional PnP formulation under ISTA-based optimization.


\section{Batch Algorithm}
\label{Sec:Convergence}

In this section, we present a detailed theoretical convergence analysis of batch PnP-ISTA. The results are based on the fixed point analysis of Algorithm~\ref{alg:pppista} and rely on basic convex and monotone analysis, summarized in Appendix~\ref{Sec:Prelims}.

The central building block of PnP-ISTA is the following denoiser-gradient operator
\begin{equation}
\label{Eq:ProxDenOp}
\Psf(\xbm) \defn \denoise_\sigma(\xbm - \gamma \nabla d(\xbm)),
\end{equation}
which first computes the gradient-step with respect to the function $d$  and then denoises the result with a given denoiser. Throughout this paper, we assume that the function $d$ is convex and has a Lipschitz continuous gradient with constant $L > 0$. We are interested in convergence of Algorithm~\ref{alg:pppista} to the set of fixed points of the operator $\Psf$
\begin{equation}
\fix(\Psf) \defn \{\xbm \in \R^n : \xbm = \Psf(\xbm)\}.
\end{equation}
Note that when $\denoise_\sigma$ is the proximal oprerator of some convex function, $\fix(\Psf)$ coincides with the set of solutions of~\eqref{Eq:RegularizedOptimization}.
\begin{proposition}
\label{Prop:Minimizer}
Let $\denoise_\sigma(\cdot) = \prox_{\gamma r}(\cdot)$ for $\gamma, \sigma > 0$. Then $\xbmast \in \fix(\Psf)$ if and only if it minimizes $f = d + r$.
\end{proposition}
\begin{proof}
See Appendix~\ref{Sec:MinimizerOfPnP}.
\end{proof}

Our central goal, however, is to generalize $\denoise_\sigma$ beyond proximal operators. The key assumption that we adopt for our analysis is that the denoiser is averaged (see Appendix A).
\begin{definition}
\label{Eq:AveragedDefinition}
Consider an operator $\denoise_\sigma$ and a constant ${\theta \in (0, 1)}$. $\denoise_\sigma$ is $\theta$-averaged if and only if the operator $(1-1/\theta)\Isf + (1/\theta)\denoise_\sigma$, where $\Isf$ denotes the identity operator, is nonexpansive.
\end{definition}
\noindent
The class of averaged operators is a superset of proximal operators and a subset of nonexpansive operators. In fact, the proximal operator is an averaged operator with $\theta = 1/2$. Note that given any nonexpansive denoiser, it is always possible to make it averaged by defining a damped operator ${\Dsf \defn (1-\theta)\Isf + \theta \denoise_\sigma}$, with $\theta \in (0, 1)$, which has the same set of fixed points as $\denoise_\sigma$~\cite{Parikh.Boyd2014}.

\begin{assumption}
\label{As:Assumption1}
We analyze PnP-ISTA under the following assumptions:
\begin{enumerate}[label=(\alph*)]
\item The function $d$ is convex and differentiable with a Lipschitz continuous gradient of constant $L > 0$.
\item $\denoise_\sigma$ is $\theta$-averaged with $\theta \in (0, 1)$ for any $\sigma > 0$.
\item There exists $\xbmast \in \R^n$ such that $\xbmast \in \fix(\Psf)$.
\end{enumerate}
\end{assumption}
\noindent
We can then establish the following convergence result.
\begin{proposition}
\label{Prop:BatchPnP}
Run PnP-ISTA for $t \geq 1$ iterations under Assumption~\ref{As:Assumption1} with the step ${\gamma \in (0, 1/L]}$ and ${q_k = 1}$ for all ${k \in \{1,\dots, t\}}$. Then, for $\xbmast \in \fix(\Psf)$, we have that
\begin{equation*}
\frac{1}{t}\sum_{k = 1}^t \|\xbm^{k-1}-\Psf(\xbm^{k-1})\|_2^2 \leq \frac{2}{t}\left(\frac{1+\theta}{1-\theta}\right) \|\xbm^0-\xbmast\|_2^2.
\end{equation*}
\end{proposition}

\begin{proof}
See Appendix~\ref{Sec:BatchPnPConvergence}.
\end{proof}

\noindent
The direct consequence of Proposition~\ref{Prop:BatchPnP} is that
\begin{equation}
\min_{k \in \{1,\dots, t\}}\left\{\|\xbm^{k-1}-\Psf(\xbm^{k-1})\|_2^2\right\} = O(1/t),
\end{equation}
that is under Assumption~\ref{As:Assumption1}, the iterates of PnP-ISTA can get arbitrarily close to the set of fixed points $\fix(\Psf)$ with rate $O(1/t)$. Note that the result is expressed in terms of the distance to $\xbm = \Psf(\xbm)$ as PnP-ISTA is not necessarily minimizing an objective function.

Recently, Meinhardt \emph{et al.}~\cite{Meinhardt.etal2017} have showed that for continuous denoisers, the fixed-points of several PnP algorithms coincide. The following proposition is a minor variation of their result tailored for PnP-ADMM.
\begin{proposition}
\label{Prop:SameFix}
Under Assumption~\ref{As:Assumption1}, the set of fixed-points of PnP-ADMM coincides with $\fix(\Psf)$.
\end{proposition}

\begin{proof}
See Appendix~\ref{Sec:SameFixedPointProof}.
\end{proof}
\noindent
In the context of the work by Sreehari \emph{et al.}~\cite{Sreehari.etal2016}, the propositions above indicate that the symmetric gradient assumption is not necessary for the convergence of PnP-ISTA. Moreover, both PnP-ISTA and PnP-ADMM are equivalent in the sense that they have the same set of solutions specified by $\fix(\Psf)$.

The bounded denoiser assumption~\eqref{Eq:BoundedDenoiser} is a more relaxed assumption on the denoising operator and was used to analyze PnP-ADMM. However, we argue that it is not sufficient to guarantee the convergence of PnP-ISTA. The following proposition builds on a specific counter example.
\begin{proposition}
\label{Prop:Divergence}
There exists a function $d$ that is convex and has a Lipschitz continuous gradient of constant $L$, and a denoiser $\denoise_\sigma$ that satisfies~\eqref{Eq:BoundedDenoiser}, such that PnP-ISTA with the step $\gamma \in (0, 1/L)$, $q_k = 1$ for all $k \in \N$, and ${\sigma > \gamma/\sqrt{c}}$ diverges.
\end{proposition}

\begin{proof}
See Appendix~\ref{Sec:Divergence}.
\end{proof}


Definition~\ref{Eq:AveragedDefinition} makes verifying that a denoiser is averaged equivalent to verifying nonexpansiveness of some operator. As was argued in several recent publications~\cite{Sreehari.etal2016, Chan.etal2016, Teodoro.etal2017} the task is more difficult for some denoisers than it is for others and there exist denoisers for which this condition does not hold. However, all recently designed denoisers for PnP from~\cite{Sreehari.etal2016, Teodoro.etal2017} satisfy our assumptions. As an example, the modified nonlocal means (NLM) filter from~\cite{Sreehari.etal2016} is by definition an averaged operator.


\section{Online Algorithm}
\label{Sec:Proposed}

We now introduce our second key contribution: the new online variant of PnP-ISTA called $\proposed$. We additionally prove its convergence for averaged denoisers.

\begin{figure}
\begin{algorithm}[H]
\caption{$\mathsf{PnP}$-$\mathsf{SGD}$}\label{alg:stochPnP}
\begin{algorithmic}[1]
\State \textbf{input: } $\xbm^0 = \sbm^0 \in \R^n$, $\gamma > 0$, $\sigma > 0$, $\{q_k\}$, and $B \geq 1$
\For{$k = 1, 2, \dots$}
\State $\nablahat d(\sbm^{k-1}) \leftarrow \mathsf{minibatchGradient}(\sbm^{k-1}, B)$
\State $\zbm^k \leftarrow \sbm^{k-1}-\gamma \nablahat d(\sbm^{k-1})$
\State $\xbm^k \leftarrow \denoise_\sigma(\zbm^k)$
\State $\sbm^k \leftarrow \xbm^k + ((q_{k-1}-1)/q_k)(\xbm^k-\xbm^{k-1}) $
\EndFor\label{euclidendwhile}
\end{algorithmic}
\end{algorithm}
\end{figure}

In many imaging applications, the data-fidelity term $d$ consists of a large number of component functions
\begin{equation}
\label{Eq:ComponentData}
d(\xbm) = \E[d_i(\xbm)] = \frac{1}{I}\sum_{i = 1}^I d_i(\xbm),
\end{equation}
where each $d_i$ typically depends only on the subset of the measurements $\ybm$. Note that in the notation~\eqref{Eq:ComponentData}, the expectation is taken over a uniformly distributed random variable ${i \in \{1,\dots, I\}}$. The computation of the gradient of $d$,
\begin{equation}
\nabla d(\xbm) = \E[\nabla d_i(\xbm)] = \frac{1}{I} \sum_{i = 1}^I \nabla d_i(\xbm),
\end{equation}
scales with the total number of components $I$, which means that when the latter is large, the classical batch PnP algorithms may become impractical in terms of speed or memory requirements. The central idea of $\proposed$, summarized in Algorithm~\ref{alg:stochPnP}, is to approximate the gradient at every iteration with an average of $B \ll I$ component gradients
\begin{equation}
\label{Eq:StochGrad}
\nablahat d(\xbm) = \frac{1}{B}\sum_{b = 1}^B \nabla d_{i_b}(\xbm),
\end{equation}
where $i_1, \dots, i_B$ are independent random variables that are distributed uniformly over $\{1, \dots, I\}$. The minibatch size parameter $B \geq 1$ controls the number of gradient components used at every iteration.
\begin{assumption}
\label{As:Assumption2}
We analyze $\proposed$ under the following assumptions:
\begin{enumerate}[label=(\alph*)]
\item The functions $d_i$ are all convex and differentiable with the same Lipschitz constant $L > 0$.
\item $\denoise_\sigma$ is $\theta$-averaged with $\theta \in (0, 1)$ for any $\sigma > 0$.
\item There exists $\xbmast \in \R^n$ such that $\xbmast \in \fix(\Psf)$.
\item At every iteration, the gradient estimate is unbiased and has a bounded variance:
$$\E[\nablahat d(\xbm)] = \nabla d(\xbm) \quad\text{and}\quad \E[\|\nabla d(\xbm)-\nablahat d(\xbm)\|_2^2] \leq \frac{\nu^2}{B},$$
for some constant $\nu > 0$.
\end{enumerate}
\end{assumption}

\noindent
Note that Assumption~\ref{As:Assumption2}(a) implies that the complete data-fidelity term $d$ is also convex and has a Lipschitz continuous gradient of constant $L$. The key difference between Assumption~\ref{As:Assumption1} and Assumption~\ref{As:Assumption2} is the last condition. The fact that the minibatch gradient is unbiased is the direct consequence of~\eqref{Eq:StochGrad}. The bounded variance assumption is a standard assumption used in the analysis of online and stochastic algorithms~\cite{Ghadimi.Lan2016, Bernstein.etal2018, Kamilov2018}.

\begin{proposition}
\label{Prop:StochPnP}
Run $\proposed$ for $t \geq 1$ iterations under Assumption~\ref{As:Assumption2} with the step ${\gamma \in (0, 1/L]}$ and ${q_k = 1}$ for all ${k \in \{1,\dots, t\}}$. Then, for $\xbmast \in \fix(\Psf)$, we have that
\begin{align*}
\E&\left[\frac{1}{t}\sum_{k = 1}^t \|\xbm^{k-1}-\Psf(\xbm^{k-1})\|_2^2\right] \\
&\leq 2\left(\frac{1+\theta}{1-\theta}\right) \left[\frac{\gamma^2\nu^2}{B} + \frac{2\gamma \nu}{\sqrt{B}}\|\xbm^0-\xbmast\|_2 + \frac{\|\xbm^0-\xbmast\|_2^2}{t}\right],
\end{align*}
where $\Psf(\cdot)$ is given by~\eqref{Eq:ProxDenOp}.
\end{proposition}
\begin{proof}
See Appendix~\ref{Sec:StochPnPConvergence}.
\end{proof}

\noindent
This result shows that the convergence in expectation of $\proposed$ to an element of $\fix(\Psf)$ is proportional to the step-size $\gamma$ and inversely proportional to the mini-batch size $B$. By controlling these two parameters, we can obtain the following convergence rates.

\begin{corollary}
\label{Cor:StochPnP}
Consider Proposition~\ref{Prop:StochPnP} with the following fixed (i.e., independent of iteration $k$) parameters.
\begin{enumerate}[label=(\alph*)]
\item For $\gamma = 1/(L\sqrt{t})$ and $B = 1$, we have that
$$\E\left[\frac{1}{t}\sum_{k = 1}^t \|\xbm^{k-1}-\Psf(\xbm^{k-1})\|_2^2\right]\leq \frac{A}{\sqrt{t}},$$
\item For $\gamma = 1/L$ and $B = t$, we have that
$$\E\left[\frac{1}{t}\sum_{k = 1}^t \|\xbm^{k-1}-\Psf(\xbm^{k-1})\|_2^2\right]\leq \frac{A}{\sqrt{t}},$$
\item For $\gamma = 1/(L\sqrt{t})$ and $B = t$, we have that
$$\E\left[\frac{1}{t}\sum_{k = 1}^t \|\xbm^{k-1}-\Psf(\xbm^{k-1})\|_2^2\right]\leq \frac{A}{t},$$
\end{enumerate}
where
$$A \defn 2\left(\frac{1+\theta}{1-\theta}\right) \left(\|\xbm^0-\xbmast\|_2 + \frac{\nu}{L}\right)^2.$$
\end{corollary}

\noindent
Corollary~\ref{Cor:StochPnP}(c) implies the worst-case convergence rate
\begin{equation}
\label{Eq:MinConv}
\E\left[\min_{k \in \{1,\dots, t\}}\left\{\|\xbm^{k-1}-\Psf(\xbm^{k-1})\|_2^2\right\}\right] = O(1/t),
\end{equation}
which means that under Assumption~\ref{As:Assumption2} and with a particular selection of parameters $B$ and $\gamma$, the iterates of $\proposed$ (in expectation) can get arbitrarily close to $\fix(\Psf)$ as $O(1/t)$.

\begin{figure}[t]
\begin{center}
\includegraphics[width=8.5cm]{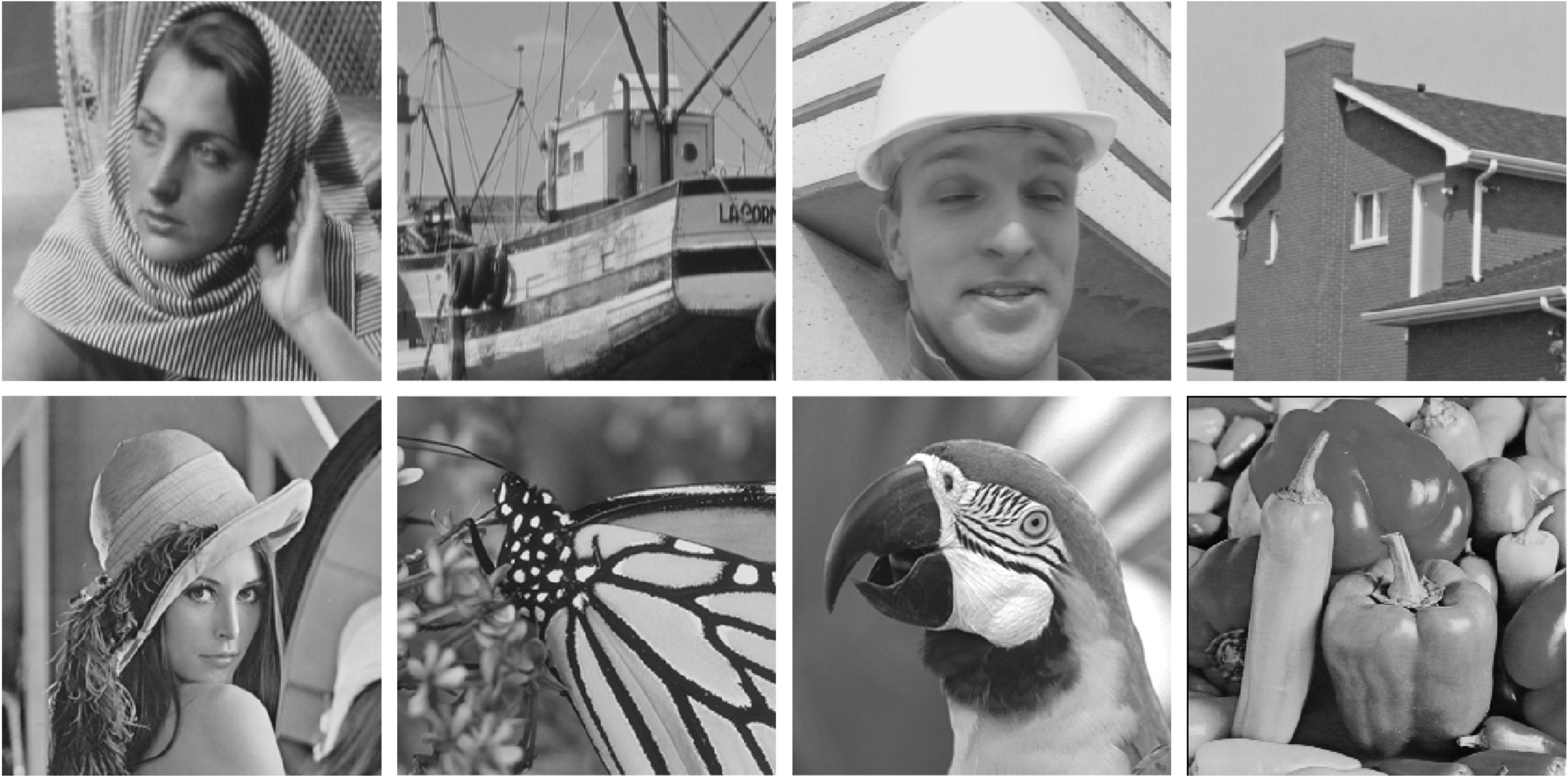}
\end{center}
\caption{Test images used. Top row from left to right: \emph{Babara}, \emph{Boat}, \emph{Foreman}, \emph{House}. Bottom row from left to right: \emph{Lenna}, \emph{Monarch}, \emph{Parrot}, \emph{Pepper}.}
\label{Fig:TestImages}
\end{figure}


\section{Numerical Simulations}
\label{Sec:Simulations}

We now empirically validate \proposed~in the context of diffraction tomography (DT) using three popular denoisers: TV~\cite{Rudin.etal1992}, BM3D~\cite{Dabov.etal2007}, and TNRD~\cite{Chen.Pock2016}. Our goal is not to justify the PnP framework, as its benefits have been well illustrated in prior work~\cite{Venkatakrishnan.etal2013, Sreehari.etal2016, Kamilov.etal2017}, but to focus on the aspects that relate to online processing of data. Therefore, we first discuss empirical convergence of $\proposed$, and then highlight the benefit of using it for processing a large number of measurements. 

\begin{figure*}[t]
\begin{center}
\includegraphics[width=16cm]{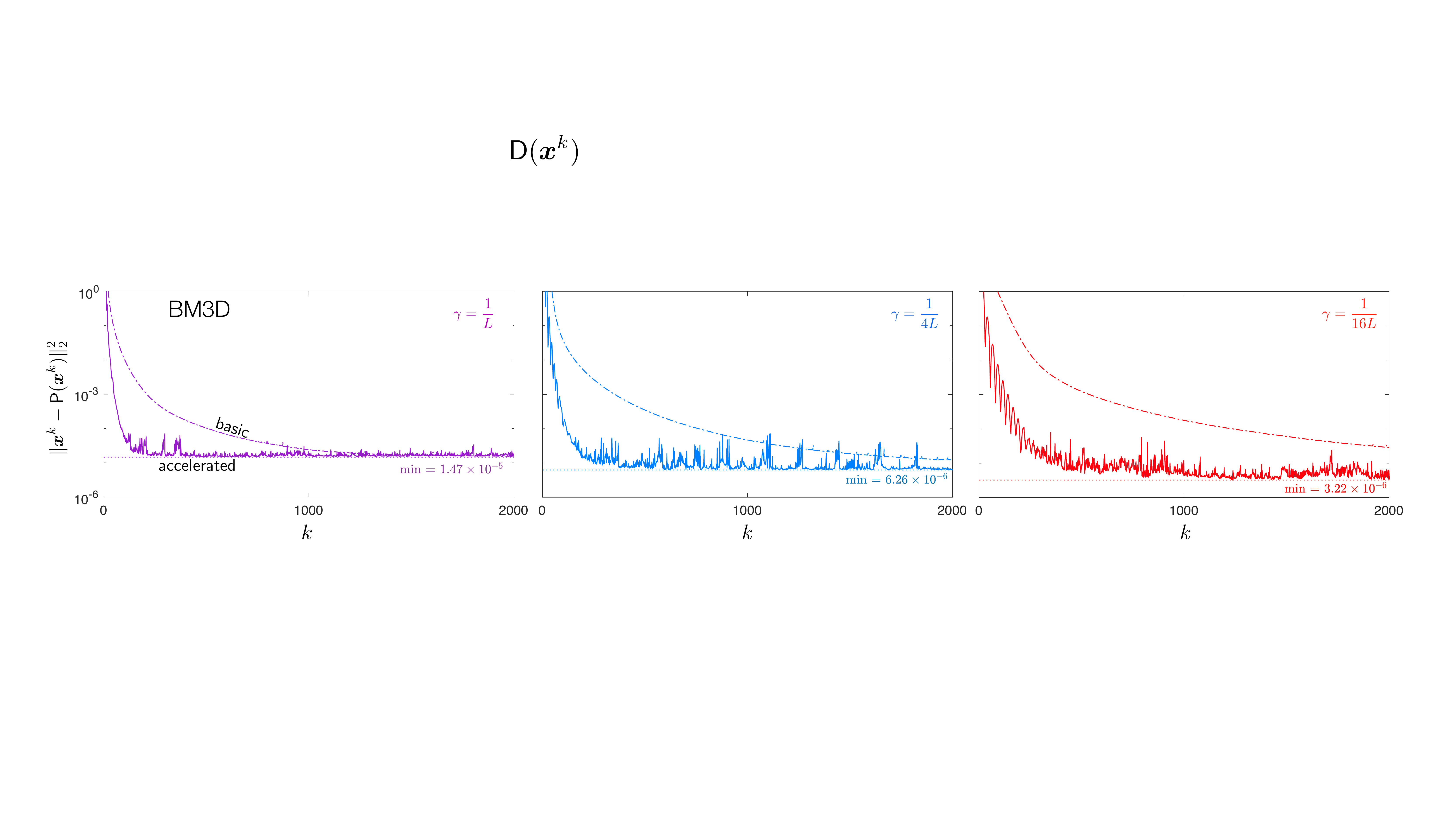}
\end{center}
\caption{Illustration of the influence of the step-size $\gamma$ on the convergence of $\proposed$ under BM3D. The distance to a fixed point is plotted against the iteration number for 3 distinct step-sizes for both accelerated (dashed) and basic (solid) variants of $\proposed$ for $B = 30$. The dotted line at the bottom shows the minimal distance to a fixed point attained by the algorithm. This plot illustrates that the empirical performance of $\proposed$ under BM3D is consistent with Proposition~\ref{Prop:StochPnP}, where the accuracy improves with smaller $\gamma$.}
\label{Fig:StepBM3D}
\end{figure*}

\begin{figure*}[t]
\begin{center}
\includegraphics[width=16cm]{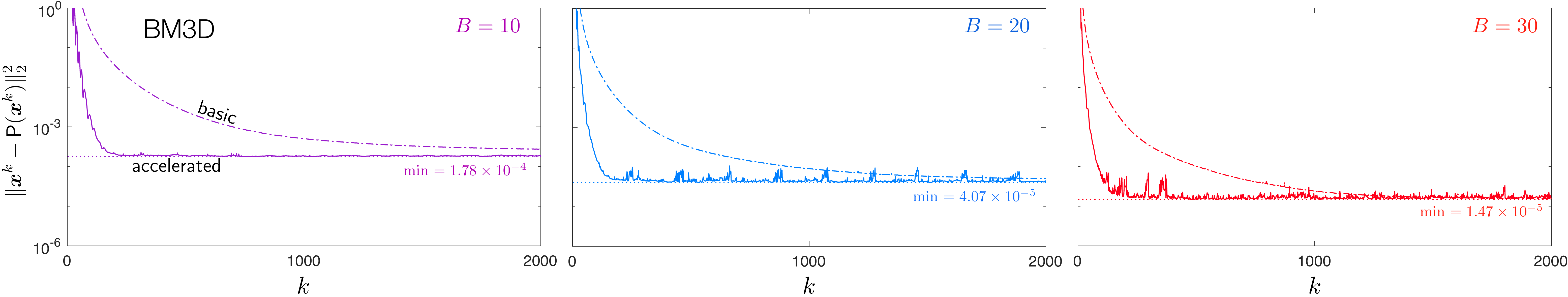}
\end{center}
\caption{Illustration of the influence of the minibatch size $B$ on the convergence of $\proposed$ under BM3D. The distance to a fixed point is plotted against the iteration number for 3 distinct minibatch sizes for both accelerated (dashed) and basic (solid) variants of $\proposed$ for $\gamma = 1/L$. The dotted line at the bottom shows the minimal distance to a fixed point attained by the algorithm. This plot illustrates that the empirical performance of $\proposed$ using BM3D is consistent with Proposition~\ref{Prop:StochPnP}, where the accuracy improves with larger $B$.}
\label{Fig:BatchBM3D}
\end{figure*}

\subsection{Diffraction tomography}

DT is a technique used to form an image of the distribution of dielectric permittivity within an object from multiple of measurements of light it scatters~\cite{Wolf1969, Kak.Slaney1988}. This problem is common in a number of applications---including ultrasound~\cite{Bronstein.etal2002} and optical microscopy~\cite{Sung.etal2009}---and is known to be highly data-intensive. A typical reconstruction task uses hundreds or thousands of measurements for forming a single image. As is common in DT, we adopt the first-Born approximation~\cite{Wolf1969}, which leads to the linear inverse problem formulation of image reconstruction.

Note that $\proposed$ is applicable beyond DT and our choice of the latter is only due to the fact that image reconstruction in DT requires the processing of a large number of distinct measurements. Additionally, our focus is not on the experimental application of DT, but rather on the demonstration of our online algorithm for image reconstruction. Hence, we restrict our study here to image reconstruction from purely simulated DT data, which enables optimal parameter tuning and quantitative comparisons. 

Consider an object with the permittivity distribution $\epsilon(\rbm)$ within a bounded domain ${\Omega \subseteq \R^2}$ with a background medium of permittivity $\epsilon_b$. The object is illuminated with a monochromatic and coherent incident electric field $\uin(\rbm)$ emitted by one of $N$ transmitters. The incident field is assumed to be known both inside $\Omega$ and at the sensor domain ${\Gamma \subseteq \R^2}$. The measurements correspond to the field scattered by the object recorded by $M$ receivers located within $\Gamma$. Under the first-Born approximation, the measurement matrix for a single illumination can be represented as $\Hbm = \Sbf \mathsf{diag}(\ubfin)$, where ${\ubfin \in \C^N}$ is the input field $\uin$ inside $\Omega$, and $\Sbf \in \C^{M \times N}$ is the discretization of the Green's function evaluated at $\Gamma$~\cite{Liu.etal2018}. In practice, the image reconstruction relies on the set of illuminations $\{\mathbf{u}_{\text{\tiny in}}^i\}_{i \in \{1, \dots, I\}}$, with each individual illumination resulting in a measurement $\ybm^i \in \C^M$ and a distinct measurement matrix $\Hbm_i$.

The objects we reconstruct correspond to the eight standard grayscale images 
shown Fig.~\ref{Fig:TestImages}. The physical size of an image is set to 18 cm $\times$ 18 cm, discretized to a grid of $256 \times 256$. The wavelength of the illumination was set to $\lambda = 0.84$ cm and the background medium was assumed to be air with $\epsilon_b = 1$. We additionally set the number of transmitters to $N = 60$, distributed uniformly along a circle of radius $1.6$ meters, and for each illumination, the corresponding scattered field is measured by $M = 360$ receivers around the object. The simulated measurements were additionally corrupted by an additive white Gaussian noise (AWGN) corresponding to 40 dB of input signal-to-noise ratio (SNR). SNR is also used as a quantitative metric for numerically evaluating the reconstruction quality in the experiments. We use the term \emph{average SNR} to indicate the SNR averaged over all the test images. In each experiment, all algorithmic hyperparameters were optimized for the best SNR performance with respect to the ground truth test image.

\begin{figure}[t]
\begin{center}
\includegraphics[width=7cm]{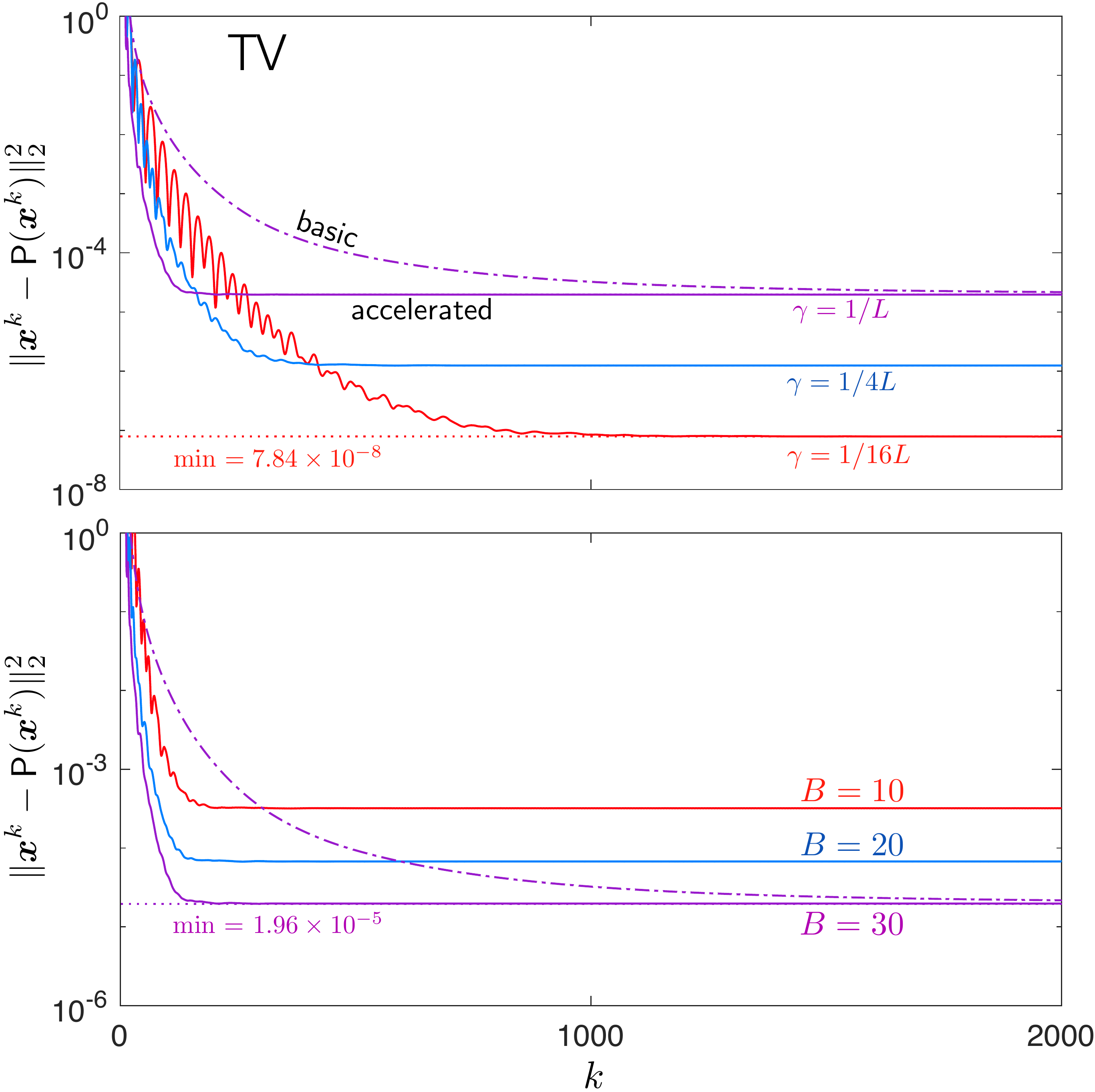}
\end{center}
\caption{Illustration of the influence of the step and minibatch sizes on the convergence of the \proposed~under TV. The dotted line at the bottom shows the minimal distance to a fixed point attained by the algorithm. A proximal operator is $1/2$-averaged, which means that it perfectly satisfies the assumptions of Proposition~\ref{Prop:StochPnP}.}
\label{Fig:TestTV}
\end{figure}

\begin{table}[t]
\centering
\scriptsize
\textbf{\caption{\label{Tab:Distance} Minimal distance averaged over the test image set}}
\begin{tabular*}{9.1cm}{L{0pt}C{24pt}C{24pt}C{24pt}|C{24pt}C{24pt}C{24pt}} \toprule
\multicolumn{1}{l}{\textbf{Denoiser}} & \multicolumn{3}{c}{\textbf{Step size} ($\gamma$)} & \multicolumn{3}{c}{\textbf{Mini-batch size} ($B$)} \\
\cmidrule{2-4} \cmidrule{5-7} 
& $1/L$ & $1/4L$ & $1/16L$ & 10 & 20 & 30   \\
\cmidrule{1-7}
\textbf{TV}   & 1.96e-5 & 1.47e-6 & 7.83e-8 & 3.18e-4 & 6.71e-5 & 1.96e-5 \\
\textbf{BM3D} & 1.47e-5 & 6.26e-6 & 3.22e-6 & 1.78e-4 & 4.07e-5 & 1.47e-5 \\
\textbf{TNRD} & 4.20e-2 & 9.18e-3 & 1.44e-3 & 3.12e-1 & 1.14e-1 & 4.20e-2 \\ \bottomrule
\end{tabular*}
\end{table}

\subsection{Convergence of PnP-SGD}

One of the key conclusions of Proposition~\ref{Prop:StochPnP} is that the final accuracy of $\proposed$~to a fixed point is proportional to the step size and inversely proportional to the minibatch size. In order to numerically evaluate the convergence, we define the distance to $\fix(\Psf)$ at the $k$th iteration as
\begin{equation}
\label{Eq:DistToFixPoint}
\mathsf{dist}(\xbm^k) \defn \| \xbm^k - \Psf(\xbm^k) \|_2^2\;,
\end{equation}
where $\Psf$ is given by~\eqref{Eq:ProxDenOp}. As the sequence $\{\xbm^k\}$ approaches $\mathsf{fix}(\Psf)$, $\mathsf{dist}(\xbm^k)$ approaches zero.

Fig.~\ref{Fig:StepBM3D} and Fig.~\ref{Fig:BatchBM3D} empirically evaluate the evolution of the distance to a fixed point for different step and minibatch sizes, respectively. $\proposed$ under BM3D is run until convergence with $\gamma \in \{1/L, 1/(4L), 1/(16L)\}$ and $B \in \{10, 20, 30\}$. Here, the quantity ${L > 0}$ denotes the Lipschitz constant, which, for linear inverse problems, corresponds to the squared largest singular value of the measurement matrix~\cite{Beck.Teboulle2009}. We show the performance of both basic and accelerated variants of $\proposed$, where the latter is obtained by setting $\{q_k\}$ as in~\eqref{Eq:Accelerate}. The plots clearly illustrate the improvement in final accuracy for smaller $\gamma$ and larger $B$, which is consistent with Proposition~\ref{Prop:StochPnP}. Additionally, they indicate that the convergence is significantly improved when using the accelerated variant of the algorithm. Note that our theoretical analysis does not predict monotonic reduction of the distance, which also seems to be consistent with the empirical performance of $\proposed$. In Fig.~\ref{Fig:TestTV}, we provide a reference plot showing the performance of $\proposed$ under TV, which is a valid proximal operator and hence is known to be a $1/2$-averaged operator. We can again observe that the convergence behavior of $\proposed$ is consistent with Proposition~\ref{Prop:StochPnP}. Finally, the summary in Table~\ref{Tab:Distance}, highlights the same convergence trends for all three algorithms, where both $\gamma$ and $B$ control the accuracy of $\proposed$.

\begin{figure}[t]
\begin{center}
\includegraphics[width=7.5cm]{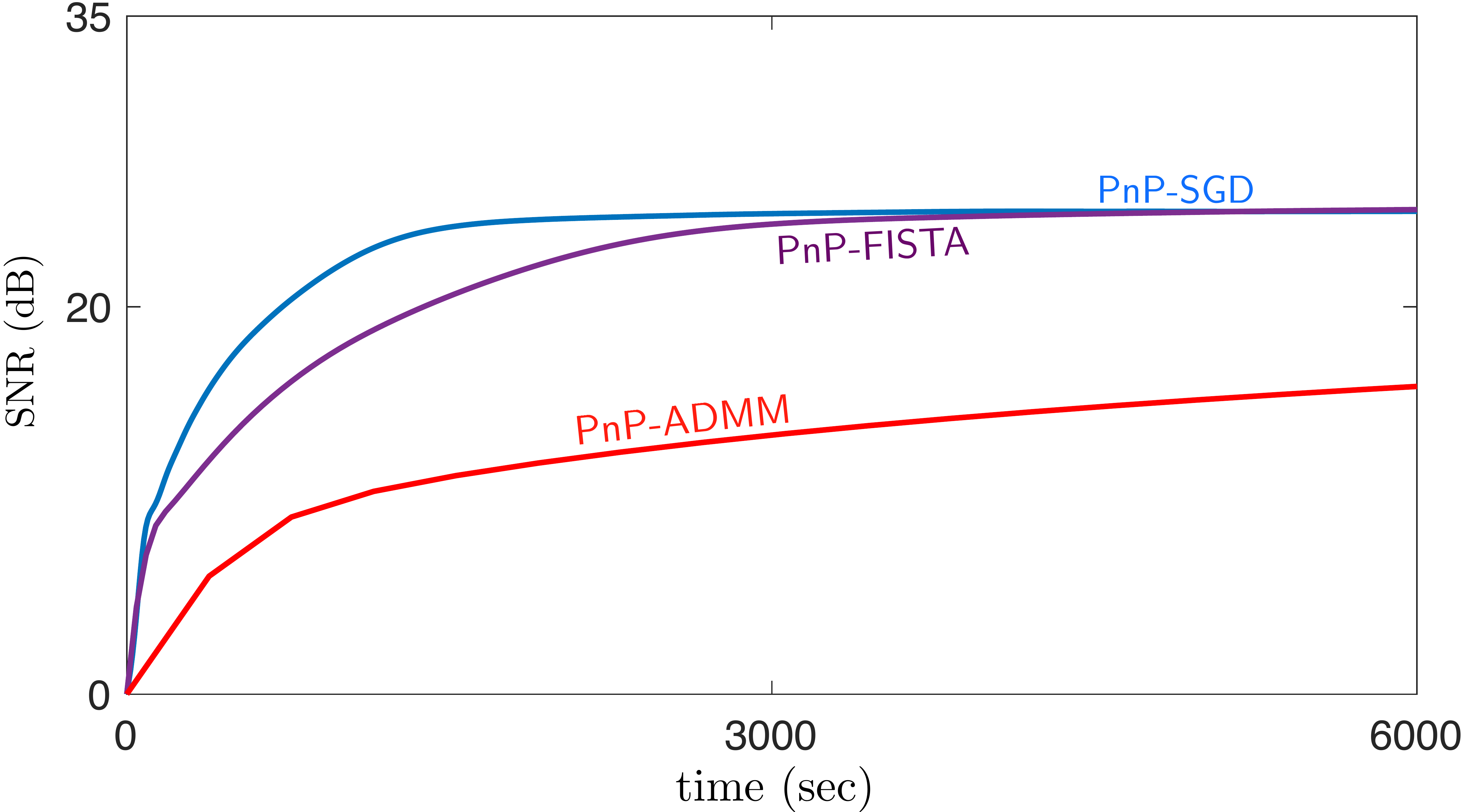}
\end{center}
\caption{Comparison between the batch and online PnP algorithms for a fixed reconstruction time. SNR (dB) is plotted against the time in seconds for three algorithms: \proposed, PnP-FISTA, and PnP-ADMM. Both PnP-FISTA and PnP-ADMM use the full set of 60 illuminations at every iteration, while $\proposed$ uses a random subset of $10$ illuminations. This lower per-iteration cost, leads to a substantially faster convergence of $\proposed$.}
\label{Fig:TimeLimit}
\end{figure}

\begin{figure}[t]
\begin{center}
\includegraphics[width=7.5cm]{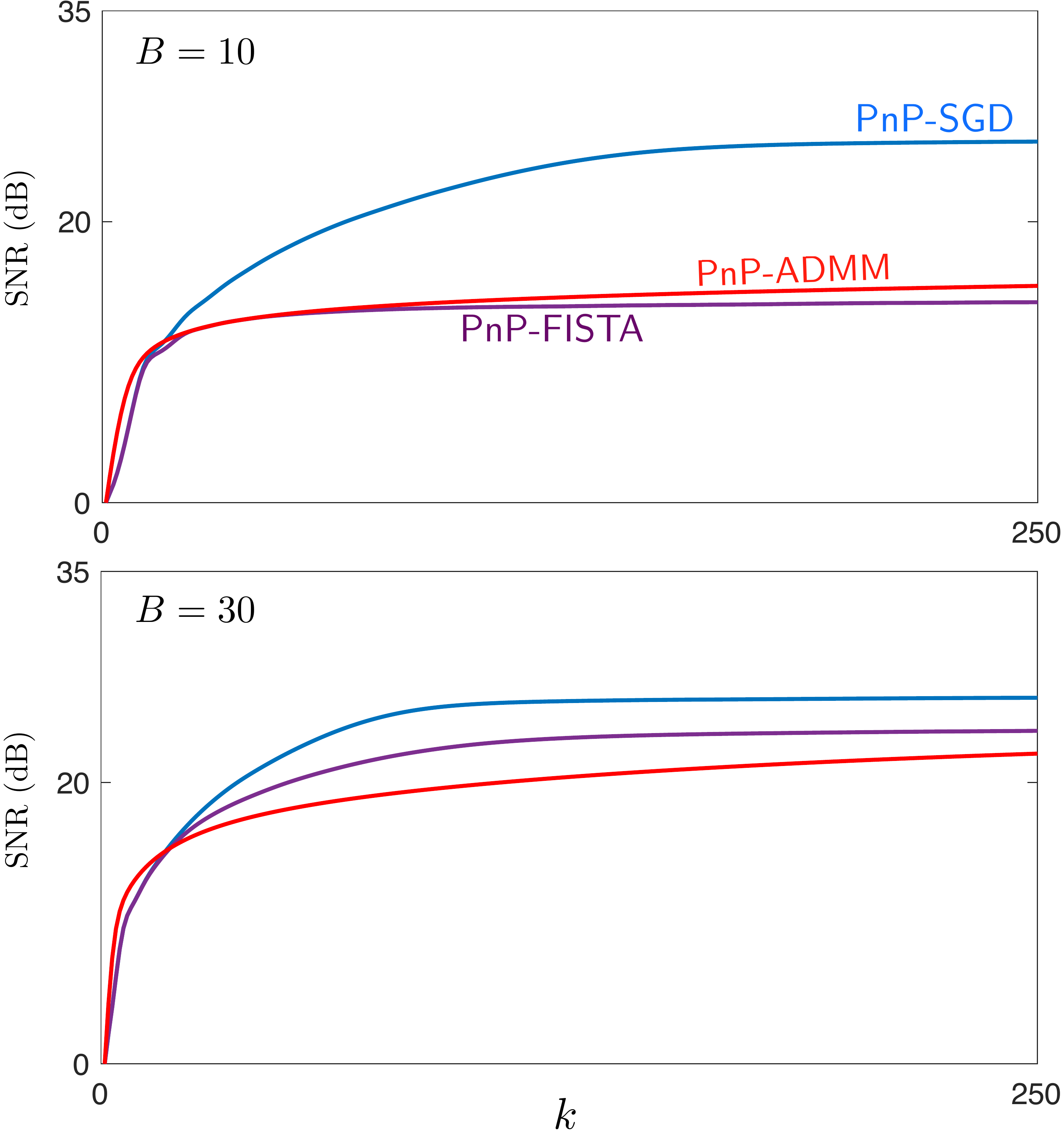}
\end{center}
\caption{Comparison between the batch and online PnP algorithms under a fixed measurement budget. SNR (dB) is plotted against the number of iterations for three algorithms: \proposed, PnP-FISTA, and PnP-ADMM. The top and bottom figures show the performance when the budget is 10 and 30 illuminations, respectively. The plot illustrates that for the same per iteration cost, $\proposed$ can significantly outperform its batch counterparts.}
\label{Fig:ComputeLimit}
\end{figure}

\begin{figure*}[t]
\begin{center}
\includegraphics[width=16cm]{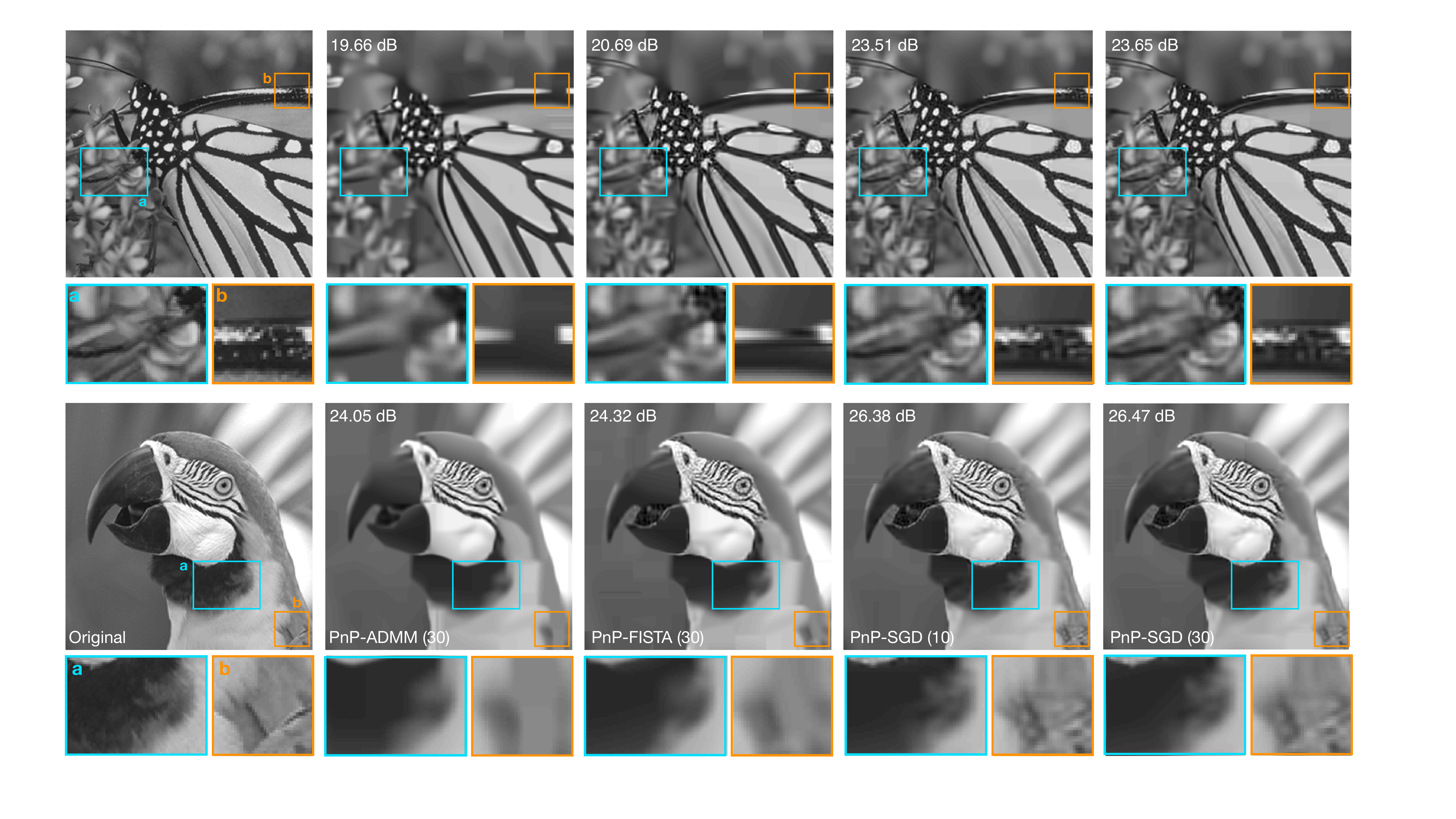}
\end{center}
\caption{Visual illustration of the reconstructed \emph{Monarch} and \emph{Parrot} images obtained using the \proposed, PnP-FISTA, and PnP-ADMM, all under BM3D. The original images are displayed in the first column. The second and the third columns show the results of PnP-FISTA and PnP-ADMM with the budget of 30 illuminations, and the fourth and the fifth columns present the results of the $\proposed$ with the budget of 10 and 30 illuminations. Visual differences are highlighted using the rectangles drawn inside the images. Each reconstruction is labeled with its SNR (dB) value with respect to the original image.}
\label{Fig:VisualExample}
\end{figure*}

\subsection{Benefits of online processing}

We now highlight the higher efficiency of $\proposed$ against PnP-ISTA and PnP-ADMM for larger number of measurements. Specifically, we consider two scenarios where: (a) the total time budget is fixed; (b) the number of measurements is fixed. While we use BM3D as our plug-in operator of choice, we note that our observations here directly generalize to any other denoiser.

Fig.~\ref{Fig:TimeLimit} compares the average reconstruction SNR of $\proposed$, PnP-FISTA, and PnP-ADMM for a fixed run-time. The batch algorithms use the full $60$ illuminations at every iteration, while $\proposed$ uses only $B = 10$ illuminations per iteration. This gives $\proposed$ a significantly lower per iteration cost compared to the batch algorithms. Specifically, the average per iteration time for \proposed, PnP-FISTA, and PnP-ADMM was 8.86 seconds, 44.94 seconds, and 382.83 seconds, respectively. The higher cost of PnP-ADMM is the result of the forward model inversion in~\eqref{Eq:ADMMDataUpdate}. This figure illustrates that, in practice, even with $B = 10$, the solution of $\proposed$ is sufficiently close to that of the batch algorithm. Additionally, $\proposed$ achieves a significant speedup due to the reduction in per-iteration complexity. This indicates to the potential of the algorithm for efficient image reconstruction from a large number of measurements.

Fig.~\ref{Fig:ComputeLimit} compares the average reconstruction SNR of $\proposed$, PnP-FISTA, and PnP-ADMM for a fixed per-iteration measurement budget. Both batch algorithm are allowed to use only 10 (top figure) or 30 (bottom figure) uniformly distributed illuminations. Similarly, $\proposed$ uses the same number of illuminations per iteration, but randomly cycles through all the measurements. This means that in each figure both $\proposed$ and PnP-FISTA have the same per-iteration computational complexity. The computational complexity of PnP-ADMM is higher due to the need to invert the measurement matrix. Table~\ref{Tab:SNR} shows the final SNR obtained by all three algorithms on each individual image in the dataset. Additionally, two visual illustrations on \emph{Monarch} and \emph{Parrot} are shown in Fig.~\ref{Fig:VisualExample}. As expected, $\proposed$ achieves dramatically higher SNR  compared to batch algorithms, since it makes use of the full set of measurements. Additionally, we note the comparable final SNR performance of $\proposed$ with $B = 10$ and $B = 30$, with the latter leading to a faster convergence speed. These results again highlight the potential of $\proposed$ for large-scale PnP image reconstruction.

To conclude this section, let us put the results here in the context of our theoretical analysis. Proposition~\ref{Prop:StochPnP} reveals that $\proposed$ converges to the same set of fixed points $\fix(\Psf)$ as PnP-ISTA and PnP-ADMM, up to a term that depends on the minibatch size $B \geq 1$. Larger $B$ leads to a higher accuracy of $\proposed$ with respect to $\fix(\Psf)$, which was empirically confirmed in Fig.~\ref{Fig:BatchBM3D}. The SNR results here additionally reveal that even with a relatively small $B$, $\proposed$ is accurate in terms of image quality. For example, in Table~\ref{Tab:SNR}, we can observe that the average SNR difference between $\proposed$ with $B = 10$ and $B = 30$ is within 0.2 dB of each other. Additionally, in Fig.~\ref{Fig:TimeLimit}, we observe that the batch and online algorithms approximately achieve the same final SNR performance. These observations suggest that while there is an order of magnitude difference in accuracy between $B = 10$ and $B = 30$ when measured in terms of the distance to a fixed point (see Fig.~\ref{Fig:BatchBM3D}), the difference is relatively mild when measured in terms of image quality (see Fig.~\ref{Fig:VisualExample}), with smaller $B$ nearly matching the image quality of the batch algorithm.

\begin{table}[t]
\centering
\scriptsize
\textbf{\caption{Individual reconstruction SNRs for each image.\label{Tab:SNR}}}
\begin{tabular*}{8.5cm}{L{22pt}C{22pt}C{22pt}C{22pt}C{22pt}C{22pt}C{22pt}} \toprule
\textbf{Images} & $\mathsf{PnP}$-$\mathsf{ADMM}$ ($\mathsf{10}$) &  $\mathsf{PnP}$-$\mathsf{ADMM}$ ($\mathsf{30}$) & $\mathsf{PnP}$-$\mathsf{FISTA}$ ($\mathsf{10}$) & $\mathsf{PnP}$-$\mathsf{FISTA}$ ($\mathsf{30}$) & $\mathsf{PnP}$-$\mathsf{SGD}$ ($\mathsf{10}$) & $\mathsf{PnP}$-$\mathsf{SGD}$ ($\mathsf{30}$)\\
\cmidrule{1-7}
Babara	& 15.62 & 21.18 & 13.32 & 20.21 & 23.61 & 23.90 \\
Boat		& 15.94 & 23.10 & 13.69 & 22.01 & 24.87 & 25.15 \\ 
Foreman	& 23.10 & 29.19 & 18.46 & 28.61 & 29.61 & 29.80 \\
House	& 19.23 & 26.43 & 15.68 & 26.79 & 28.29 & 28.41 \\
Lenna	& 15.52 & 23.17 & 13.49 & 22.91 & 25.30 & 25.38 \\
Monarch	& 11.46 & 19.66 & 8.80   & 20.69 & 23.51 & 23.65 \\
Parrot	& 17.29 & 24.05 & 13.72 & 24.32 & 26.38 & 26.47 \\
Pepper	& 15.49 & 22.90 & 11.68 & 22.96 & 24.92 & 25.15 \\
\cmidrule{1-7}
\textbf{Average} & 16.71 & 23.71 & 14.26 & 23.73 & 25.85 & 26.04 \\
\bottomrule
\end{tabular*}
\end{table}


\section{Conclusion}
\label{Sec:Conclusion}

The online PnP algorithm developed in this paper is beneficial in the context of large-scale image reconstruction, when the amount of data is too large to be processed jointly. We presented an in-depth theoretical convergence analysis for both batch and online variants of PnP-ISTA. Our work represents a substantial extension of the current convergence theory of PnP-algorithms for image reconstruction. Related experiments are also presented to empirically confirm the proposed propositions and to elucidate the higher efficiency of \proposed~in different representative situations. Future work will aim to apply the algorithm to other image reconstruction tasks, relax some of the assumptions, and extend the theoretical results in this paper to ADMM and FISTA.


\section{Appendix}

\subsection{Preliminaries}
\label{Sec:Prelims}

We start by reviewing the key concepts useful for our analysis. A more complete description of these ideas can be found in literature~\cite{Bauschke.Combettes2010, Parikh.Boyd2014, Ryu.Boyd2016}. 

We will represent denoisers as functions $\Dsf_\sigma: \R^n \rightarrow \R^n$ that depend on $\sigma > 0$. We will also use a shorthand notation $\Gsf_\gamma \defn \Isf - \gamma \nabla d$ to denote the gradient-step operator, where $\Isf$ denotes the identity operator. We will assume that all operators are defined everywhere on $\R^n$.
\begin{definition}
An operator $\Fsf$ is Lipschitz continuous with a constant $L > 0$ if
\begin{equation}
\|\Fsf(\xbm)-\Fsf(\ybm)\|_2 \leq L \|\xbm - \ybm\|_2, \quad \forall \xbm, \ybm \in \R^n.
\end{equation}
When $L = 1$, $\Fsf$ is said to be nonexpansive.
\end{definition}

\noindent
It is straightforward to show that given two operators $\Fsf_1$ and $\Fsf_2$ with Lipschitz constants $L_1$ and $L_2$, respectively, the composition $\Fsf \defn \Fsf_2 \circ \Fsf_1$ has Lipschitz constant $L = L_1L_2$. This means that the composition of two nonexpansive operators is also nonexpansive.

\begin{definition}
We say that $\xbmast \in \R^n$ is a fixed point of $\Fsf$ is $\xbmast = \Fsf(\xbmast)$. We denote the set of fixed points of an operator $\Fsf$ as $\fix(\Fsf) \defn \{\xbm \in \R^n: \xbm = \Fsf(\xbm)\}$.
\end{definition} 

\noindent
Note that the iteration of a nonexpansive operator does not necessarily converge. To see this consider a nonexpansive operator $\Fsf = - \Isf$, where $\Isf$ is the identity. However, the Krasnosel'skii-Mann theorem (see Theorem 5.15 in~\cite{Bauschke.Combettes2010}) states that the iteration of the damped operator ${\Dsf \defn (1-\alpha)\Isf + \alpha \Fsf}$, for $\alpha \in (0, 1)$, will converge to $\fix(\Fsf)$. This idea is further formalized with the definition of the following class of operators.
\begin{definition}
For a constant $\alpha \in (0, 1)$, we say that the operator $\Dsf$ is $\alpha$-averaged, if there exists a nonexpansive operator $\Fsf$ such that $\Dsf = (1-\alpha)\Isf + \alpha \Fsf.$
\end{definition}

\noindent
An important result from convex analysis is that the proximal operator is $(1/2)$-averaged (see p.~132 in~\cite{Parikh.Boyd2014}). Similarly, when $d$ is convex and has a Lipschitz continuous gradient of constant $L$, the gradient-step operator $\Gsf_\gamma$ is $(\gamma L/2)$-averaged for any ${\gamma \in (0, 2/L)}$ (see p.~17 in~\cite{Ryu.Boyd2016}). As stated next, the composition of two averaged operators is also averaged.
\begin{proposition}
\label{Thm:Composition}
Let $\Fsf_1$ be $\alpha_1$-averaged and $\Fsf_2$ be $\alpha_2$-averaged. Then, the composite operator $\Fsf \defn \Fsf_2 \circ \Fsf_1 = \Fsf_2 \Fsf_1$ is
\begin{equation}
\label{Eq:Composition}
\alpha \defn \frac{\alpha_1 + \alpha_2 - 2\alpha_1\alpha_2}{1-\alpha_1\alpha_2}
\end{equation}
averaged operator.
\end{proposition}

\begin{proof}
See Proposition 4.44 in~\cite{Bauschke.Combettes2010}.
\end{proof}
\noindent
The direct consequence of this theorem, is that the composition of the proximal operator and the gradient-step is also an averaged operator. The following classical result was used in Definition~\ref{Eq:AveragedDefinition} and is central for our subsequent analysis.
\begin{proposition}
\label{Prop:AveragedOp}
For a nonexpansive operator $\Dsf$ and a constant $\alpha \in (0, 1)$, the following are equivalent:
\begin{enumerate}[label=(\alph*)]
\item $\Dsf$ is $\alpha$-averaged.
\item $(1-1/\alpha)\Isf + (1/\alpha)\Dsf$ is nonexpansive.
\item For all $\xbm, \ybm \in \R^n$, we have that
\begin{align*}
\|\Dsf&(\xbm)-\Dsf(\ybm)\|_2^2 \\
&\leq \|\xbm-\ybm\|_2^2 - \left(\frac{1-\alpha}{\alpha}\right)\|\xbm-\Dsf(\xbm)-\ybm+\Dsf(\ybm)\|_2^2
\end{align*}
\end{enumerate}
\end{proposition}

\begin{proof}
See Proposition~4.35 in~\cite{Bauschke.Combettes2010}.
\end{proof}

\subsection{Proof of Proposition~\ref{Prop:Minimizer}}
\label{Sec:MinimizerOfPnP}

Proposition~\ref{Prop:Minimizer} is a direct consequence of the well-known fixed-point interpretation of ISTA (see p.~150 in~\cite{Parikh.Boyd2014}). We provide the proof here for completeness by using the following characterization of the proximal operator
\begin{equation}
\label{Eq:ProximalCharacterization}
\xbm = \prox_{\gamma r}(\zbm) \quad\Leftrightarrow\quad \frac{\zbm-\xbm}{\gamma} \in \partial r(\xbm),
\end{equation}
valid for all $\zbm \in \R^n$, where $\partial r(\xbm)$ is the subdifferential of $r$ at $\xbm$~\cite{Boyd.Vandenberghe2008}.
Let ${\denoise_\sigma(\cdot) = \prox_{\gamma r}(\cdot)}$ and $\xbmast \in \fix(\Psf)$. Then, from~\eqref{Eq:ProximalCharacterization}, we have that
\begin{align*}
&\xbmast = \Psf(\xbmast) = \prox_{\gamma r}(\xbmast - \gamma \nabla d(\xbmast))\\
&\quad\Leftrightarrow\quad -\nabla d(\xbmast) \in \partial r(\xbmast) \\
&\quad\Leftrightarrow\quad \zerobm \in \nabla d(\xbmast) + \partial r(\xbmast),
\end{align*}
which establishes the desired result.

\subsection{Proof of Proposition~\ref{Prop:BatchPnP}}
\label{Sec:BatchPnPConvergence}

As mentioned in Appendix~\ref{Sec:Prelims}, the iterative application of an averaged operator is well known as Krasnosel'skii-Mann iteration~\cite{Mann1953, Kasnoselskii1955} and its convergence has been extensively discussed in literature~\cite{Bauschke.Combettes2010, Ryu.Boyd2016}. Below, we use this theory to establish a novel convergence result for PnP-ISTA.

From our assumptions, the denoiser $\Dsf_\sigma$ is $\theta$-averaged and the gradient-step operator $\Gsf_\gamma$ is $(\gamma L/2)$-averaged for any $\gamma \in (0, 2/L)$. From Proposition~\ref{Thm:Composition}, we have that their composition ${\Psf = \Dsf_\sigma \circ \Gsf_\gamma}$ is
$$\alpha = \frac{\theta + \frac{\gamma L}{2} - \theta \gamma L}{1-\frac{\theta\gamma L}{2}}$$
averaged. Consider a single iteration ${\xbm^+ = \Psf(\xbm)}$, then we have for any $\xbmast \in \fix(\Psf)$ that
\begin{align*}
\|&\xbm^+ - \xbmast\|_2^2 = \|\Psf(\xbm)-\Psf(\xbmast)\|_2^2 \\
&\leq \|\xbm-\xbmast\|_2^2 - \left(\frac{1-\alpha}{\alpha}\right)\|\xbm-\Psf(\xbm)-\xbmast+\Psf(\xbmast)\|_2^2 \\
&= \|\xbm-\xbmast\|_2^2 - \left(\frac{1-\alpha}{\alpha}\right)\|\xbm-\Psf(\xbm)\|_2^2,
\end{align*}
where we used Proposition~\ref{Prop:AveragedOp}(c) and the fact that $\xbmast = \Psf(\xbmast)$. By considering the iteration $k \geq 1$ and rearranging the terms, we obtain
\begin{align*}
\|\xbm^{k-1}&-\Psf(\xbm^{k-1})\|_2^2 \\
&\leq \left(\frac{\alpha}{1-\alpha}\right) \left[\|\xbm^{k-1}-\xbmast\|_2^2 - \|\xbm^k-\xbmast\|_2^2\right].
\end{align*}
By averaging this inequality over $t \geq 1$ iterations and dropping the last term $\|\xbm^t-\xbmast\|_2^2$, we obtain
$$\frac{1}{t}\sum_{k = 1}^t \|\xbm^{k-1}-\Psf(\xbm^{k-1})\|_2^2 \leq \frac{1}{t}\left(\frac{\alpha}{1-\alpha}\right)\|\xbm^0-\xbmast\|_2^2.$$
To obtain the result that depends on $\theta \in (0, 1)$, we note that for any $\gamma \in (0, 1/L]$, we can write
\begin{align}
\label{Eq:ConstantBound}
\frac{\alpha}{1-\alpha} = \frac{\theta + \frac{\gamma L}{2}-\theta \gamma L}{(1-\theta)(1-\frac{\gamma L}{2})} \leq \frac{\theta + \frac{1}{2}}{\frac{1-\theta}{2}} \leq 2 \left(\frac{1+\theta}{1-\theta}\right).
\end{align}
This establishes the desired result.

\subsection{Proof of Proposition~\ref{Prop:SameFix}}
\label{Sec:SameFixedPointProof}

Proposition~\ref{Prop:SameFix} is a variation of the result in~\cite{Meinhardt.etal2017}. For completeness, we provide a proof based on the fixed-point interpretation of ADMM (see p.~157 in~\cite{Parikh.Boyd2014}). 

First note that both $\Dsf_\sigma$ and $\prox_{\gamma d}$ are continuous (since they are nonexpansive). Fixed points $\xbmast, \zbmast, \sbmast$ of PnP-ADMM satisfy
\begin{subequations}
\begin{align}
\label{Eq:PnPADMM1}&\zbmast = \prox_{\gamma d}(\xbmast - \sbmast) \\
\label{Eq:PnPADMM2}&\xbmast = \Dsf_\sigma(\zbmast + \sbmast) \\
\label{Eq:PnPADMM3}&\sbmast = \sbmast + \zbmast - \xbmast.
\end{align}
\end{subequations}
From~\eqref{Eq:PnPADMM3}, we conclude that $\zbmast = \xbmast$. By using the smoothness of $d$ and the characterization~\eqref{Eq:ProximalCharacterization} in~\eqref{Eq:PnPADMM1}, we obtain
$$\xbmast - \sbmast -\zbmast = \gamma \nabla d(\zbmast) \quad\Rightarrow\quad \sbmast = -\gamma \nabla d(\xbmast).$$
Finally, by using this in~\eqref{Eq:PnPADMM2}, we obtain
$$\xbmast = \Dsf_\sigma(\xbmast - \gamma \nabla d(\xbmast)) = \Psf(\xbmast),$$
which means that $\xbmast = \zbmast \in \fix(\Psf)$ and completes the proof.

\subsection{Proof of Proposition~\ref{Prop:Divergence}}
\label{Sec:Divergence}

We prove by providing a specific counter example. For simplicity, we assume $n = 1$, but the same example can be generalized for any $n \in \N$. Consider the data fidelity given by the Huber function
\begin{equation}
d(x) \defn 
\begin{cases} 
\frac{1}{2}x^2 & \text{if } |x| \leq 1 \\
|x|-\frac{1}{2}       & \text{if } |x| > 1
\end{cases}.
\end{equation}
This function is convex and has a Lipschitz continuous gradient with constant $L = 1$
\begin{equation}
d'(x) = 
\begin{cases} 
x & \text{if } |x| \leq 1 \\
\sgn(x)       & \text{if } |x| > 1
\end{cases},
\end{equation}
where $\sgn(\cdot)$ denotes the sign function. We also consider the denoiser defined as
\begin{equation}
\Dsf_\sigma(z) \defn z + \sigma \sqrt{c} \,\sgn(z),
\end{equation}
where $c > 0$ is some constant independent of $\sigma > 0$. Since
\begin{equation}
|\Dsf_\sigma(x) - x|^2 = \sigma^2 \, c,
\end{equation}
this denoiser satisfies the definition of boundedness in~\eqref{Eq:BoundedDenoiser}.
Then, for $q_k = 1$, a single iteration of PnP-ISTA can be re-written as
\begin{align*}
&x = \Dsf_\sigma(z) = z + \sigma\sqrt{c} \,\sgn(z) \\
&z^+ = x - \gamma d'(x) =
\begin{cases} 
(1-\gamma)x & \text{if } |x| \leq 1 \\
x-\gamma\sgn(x)       & \text{if } |x| > 1
\end{cases},
\end{align*}
where we assume any $\gamma \in (0, 1)$. By combining these equations, we obtain
\begin{equation*}
z^+ = 
\begin{cases} 
(1-\gamma)(|z|+\sigma \sqrt{c})\sgn(z) & \text{if } |z| \leq 1-\sigma \sqrt{c} \\
(|z| + \sigma \sqrt{c} - \gamma) \sgn(z)       & \text{if } |z| > 1-\sigma \sqrt{c},
\end{cases}
\end{equation*}
where we used the fact that $\sgn(x) = \sgn(z)$ and expressed ${z = |z| \, \sgn(z)}$. For $|z| \leq 1-\sigma\sqrt{c}$, we have that
\begin{align*}
|z^+| &= (1-\gamma)(|z|+\sigma\,\sqrt{c}) \\
&= |z| +\sigma \, \sqrt{c} - \gamma |z| - \gamma \sigma \, \sqrt{c} \\
&\geq |z| + \sigma \, \sqrt{c} - \gamma (1- \sigma \, \sqrt{c}) - \gamma \sigma \, \sqrt{c} \\
&= |z| + \sigma\, \sqrt{c} - \gamma.
\end{align*}
On the other hand, for $|z| > 1- \sigma \, \sqrt{c}$, we have that
$$|z^+| = |z| + \sigma \, \sqrt{c} - \gamma.$$
This means that the iterates of PnP-ISTA satisfy
$$|z^t| \geq  |z^0| + t(\sigma \, \sqrt{c} - \gamma), \quad\forall t \in \N.$$
Therefore, for any $\sigma > \gamma/\sqrt{c}$ and any $z^0 \in \R$, the sequence $\{z^t\}_{t \in \N}$ generated by PnP-ISTA diverges. Since the denoiser is bounded, this implies that the sequence $\{x^t\}_{t \in \N}$ also diverges. This completes the proof.

\subsection{Proof of Proposition~\ref{Prop:StochPnP}}
\label{Sec:StochPnPConvergence}

We define the full proximal-gradient operator
\begin{equation}
\Psf(\xbm) \defn \Dsf_\sigma(\xbm-\gamma \nabla d(\xbm))
\end{equation}
and its online variant over a minibatch of size ${B \geq 1}$
\begin{equation}
\Psfhat(\xbm) \defn \Dsf_\sigma(\xbm-\gamma \nablahat d(\xbm)),
\end{equation}
where $\nablahat d$ denotes the minibatch gradient. The variance bound in Assumption~\ref{As:Assumption2}(d) implies that for all $\xbm \in \R^n$, we have that
\begin{align}
\label{Eq:ProxDenVariance}
\nonumber\E&\left[\|\Psf(\xbm)-\Psfhat(\xbm)\|_2^2\right] \\
\nonumber&= \E\left[\|\Dsf_\sigma(\xbm-\gamma \nabla d(\xbm))-\Dsf_\sigma(\xbm-\gamma \nablahat d(\xbm))\|_2^2\right] \\
\nonumber&\leq \E\left[\|\xbm-\gamma \nabla d(\xbm)-\xbm+\gamma \nablahat d(\xbm)\|_2^2\right] \\
&\leq \gamma^2 \E\left[\|\nabla d(\xbm)-\nablahat d(\xbm)\|_2^2\right] \leq \frac{\gamma^2 \nu^2}{B},
\end{align}
where in the third row we used the nonexpansiveness of $\Dsf_\sigma$. Consider a single iteration $\xbm^k = \Psfhat(\xbm^{k-1})$, then we have for any $\xbmast \in \fix(\Psf)$ that
\begin{align}
\label{Eq:FistStochBound}
\nonumber&\|\xbm^k - \xbmast\|_2^2 = \|\Psfhat(\xbm^{k-1})-\Psf(\xbm^{k-1})+\Psf(\xbm^{k-1})-\Psf(\xbmast)\|_2^2 \\
&= \|\Psf(\xbm^{k-1})-\Psf(\xbmast)\|_2^2 + \|\Psfhat(\xbm^{k-1})-\Psf(\xbm^{k-1})\|_2^2 \\
\nonumber&\quad\quad + 2(\Psfhat(\xbm^{k-1})-\Psf(\xbm^{k-1}))^\Tsf(\Psf(\xbm^{k-1})-\Psf(\xbmast)) \\
\nonumber&\leq \|\xbm^{k-1}-\xbmast\|_2^2 - \left(\frac{1-\alpha}{\alpha}\right)\|\xbm^{k-1}-\Psf(\xbm^{k-1})\|_2^2 \\
\nonumber&\quad\quad + \|\Psfhat(\xbm^{k-1})-\Psf(\xbm^{k-1})\|_2^2 \\
\nonumber&\quad\quad + 2\|\Psfhat(\xbm^{k-1})-\Psf(\xbm^{k-1})\|_2 \cdot \|\Psf(\xbm^{k-1})-\Psf(\xbmast)\|_2,
\end{align}
where we used Proposition~\ref{Prop:AveragedOp}(c) and the Cauchy-Schwarz inequality. Note that due to nonexpansiveness of the operator $\Psf$, we have that
\begin{equation}
\label{Eq:NonexpansiveFromInitial}
\|\Psf(\xbm^{k-1})-\Psf(\xbmast)\|_2 \leq \|\xbm^{k-1}-\xbmast\|_2 \leq \|\xbm^0-\xbmast\|_2.
\end{equation}
Additionally, by applying Jensen's inequality to~\eqref{Eq:ProxDenVariance}, we conclude that for all $\xbm \in \R^n$
\begin{align}
\label{Eq:JensenSimplification}
\E&\left[\|\Psf(\xbm)-\Psfhat(\xbm)\|_2\right] = \E\left[\sqrt{\|\Psf(\xbm)-\Psfhat(\xbm)\|_2^2}\right] \\
&\leq \sqrt{\E\left[\|\Psf(\xbm)-\Psfhat(\xbm)\|_2^2\right]} \leq \frac{\gamma \nu}{\sqrt{B}}.
\end{align}
By taking a conditional expectation of~\eqref{Eq:FistStochBound} and using these bounds, we obtain
\begin{align*}
\E&\left[\|\xbm^k-\xbmast\|_2^2 - \|\xbm^{k-1}-\xbmast\|_2^2 \mid \xbm^{k-1}\right] \\
\nonumber&\leq \left(\frac{\alpha - 1}{\alpha}\right)\|\xbm^{k-1}-\Psf(\xbm^{k-1})\|_2^2 \\
\nonumber&\quad\quad+ \frac{2\gamma \nu}{\sqrt{B}}\|\xbm^0-\xbmast\|_2 + \frac{\gamma^2 \nu^2}{B},
\end{align*}
which can be rearanged into
\begin{align*}
&\|\xbm^{k-1}-\Psf(\xbm^{t-1})\|_2^2 \\
&\leq \left(\frac{\alpha}{1-\alpha}\right)\Big[\frac{\gamma^2\nu^2}{B} + \frac{2\gamma\nu}{\sqrt{B}}\|\xbm^0-\xbmast\|_2  \\
&\quad\quad+\E\left[\|\xbm^{k-1}-\xbmast\|_2^2 - \|\xbm^k-\xbmast\|_2^2 \mid \xbm^{k-1}\right]\Big].
\end{align*}
By averaging the inequality over $t \geq 1$ iterations, taking the total expectation, and dropping the last term, we obtain
\begin{align*}
\E&\left[\frac{1}{t}\sum_{k = 1}^t \|\xbm^{k-1}-\Psf(\xbm^{k-1})\|_2^2\right] \\
&\leq \frac{\alpha}{1-\alpha} \left[\frac{\gamma^2 \nu^2}{B} + \frac{2\gamma \nu }{\sqrt{B}}\|\xbm^0-\xbmast\|_2 + \frac{\|\xbm^0-\xbmast\|_2^2}{t}\right],
\end{align*}
where we used the law of total expectation.
By using the inequality~\eqref{Eq:ConstantBound}, we can rewrite this expression as
\begin{align*}
\E&\left[\frac{1}{t}\sum_{k = 1}^t \|\xbm^{k-1}-\Psf(\xbm^{k-1})\|_2^2\right] \\
&\leq 2\left(\frac{1+\theta}{1-\theta}\right) \left[\frac{\gamma^2\nu^2}{B} + \frac{2\gamma \nu}{\sqrt{B}}\|\xbm^0-\xbmast\|_2 + \frac{\|\xbm^0-\xbmast\|_2^2}{t}\right]
\end{align*}
Note that to obtain the results in Corollary~\ref{Cor:StochPnP}, simply replace given values for $\gamma$ and $B$ into the inequality, and use the following bounds that are valid for any $t \in \N$
$$\frac{1}{t}\leq \frac{1}{\sqrt{t}} \quad\text{and}\quad \frac{1}{t^2} \leq \frac{1}{t}.$$
This establishes the desired results.


\bibliographystyle{IEEEtran}

\end{document}